\documentclass{article}

\usepackage{times}
\usepackage{graphicx} % more modern
\usepackage{subfigure} 

\usepackage{natbib}

\usepackage{algorithm}
\usepackage{algorithmic}

\usepackage{hyperref}

\usepackage[non]{nonicml}

\icmltitlerunning{Unsupervised Ensemble Regression}
\usepackage{numprint}
\usepackage{amsmath,amssymb,amsthm}
\usepackage{color,cancel}
\usepackage{enumitem}
\usepackage[flushleft]{threeparttable}

\newcommand{\E}{\mathbb{E}}
\newcommand{\R}{\mathbb{R}}
\newcommand{\var}{\mathrm{Var}}

\newcommand{\1}{\mathbf 1}
\newcommand{\muY}{\theta_1}
\newcommand{\mtwo}{\theta_2}
\newcommand{\w}{\mathbf{w}}
\newcommand{\rb}{\boldsymbol{\rho}}
\newcommand{\ab}{\boldsymbol{a}}
\newcommand{\dor}{\delta_{\mbox{\tiny or}}}

\DeclareMathOperator*{\argmin}{arg\,min}
\DeclareMathOperator*{\argmax}{arg\,max}

\newtheorem{theorem}{Theorem}
\newtheorem{remark}{Remark}
\newtheorem{lemma}{Lemma}

\begin{document}

\twocolumn[
\icmltitle{Unsupervised Ensemble Regression}

\icmlauthor{Omer Dror}{omer.dror@weizmann.ac.il}
\icmladdress{Weizmann Institute of Science}
\icmlauthor{Boaz Nadler}{boaz.nadler@weizmann.ac.il}
\icmladdress{Weizmann Institute of Science}
\icmlauthor{Erhan Bilal}{ebilal@us.ibm.com}
\icmladdress{IBM Thomas J. Watson Research Center}
\icmlauthor{Yuval Kluger}{yuval.kluger@yale.edu}
\icmladdress{Yale University School of Medicine}

\icmlkeywords{unsupervised ensemble learning, ensemble methods, ensemble regression, machine learning}

\vskip 0.3in
]

%%%%%%%%%%%%%%%%%%%%%%% SECTION %%%%%%%%%%%%    
\begin{abstract}
\label{sec:abstract}
Consider a regression problem where there is no labeled data and the only observations are the predictions \(f_i(x_j)\) of \(m\) experts \(f_{i}\) over many samples \(x_j\). With no knowledge on the accuracy of the experts, is it still possible to accurately estimate the unknown responses \(y_{j}\)? Can one still detect the least or most accurate experts?  
In this work we propose a framework to study these questions, based on the assumption that the \(m\) experts have uncorrelated deviations 
from
the optimal predictor. Assuming the first two moments of the response are known, we develop methods to detect the best and worst regressors, and derive U-PCR, a novel principal components approach
for unsupervised ensemble regression. 
We provide theoretical 
support  for U-PCR and illustrate its improved accuracy over the ensemble mean and median on a variety of regression problems.  
\end{abstract}

%%%%%%%%%%%%%%%%%%%%%%% SECTION %%%%%%%%%%%%    
\section{Introduction}
\label{sec:intro}

Consider the following {\em unsupervised ensemble regression} setup:\ The only observations are an $m\times n$ matrix of real-valued predictions $f_i(x_j)$ made by $m$ different regressors or experts $\{f_i\}_{i=1}^m$, on a  set of unlabeled samples $\{x_j\}_{j=1}^n$. There is no a-priori knowledge on the accuracy of the experts and no labeled data to estimate it.   Given only the above observed data and minimal knowledge about the unobserved response, such as its mean and variance,
is it possible to (i) {\em rank} the $m$ regressors, say by their mean squared error; or at least detect the most and least accurate ones? and (ii) construct an ensemble predictor for the unobserved continuous responses \(y_j\), more accurate than both the individual predictors and simple ensemble strategies such as their mean or median?

Our motivation for studying this problem comes from several application domains, where such scenarios naturally arise. Two
such domains are biology and medicine, where in recent years
there are extensive collaborative efforts to solve challenging prediction
problems, see for example the past and ongoing  DREAM competitions\footnote{\texttt{www.dreamchallenges.org}}. Here, multiple participants construct prediction models based on published labeled data, which are then evaluated on held-out data whose statistical distribution may differ significantly from the training one. A key question is whether one can provide more accurate answers than those of the individual participants, by cleverly combining their prediction models. In the experiment section \ref{sec:experiments} we present one such example, where competitors had to predict the concentrations of multiple phosphoproteins in various cancer cell lines \cite{hill2016inferring}.
Understanding the causal relationships between these proteins is important as it may explain variation in disease phenotypes or therapeutic response \cite{hill2016context}.
A second application comes from regression problems in computer vision. A specific example, also described in Section \ref{sec:experiments}, is accurate estimation of the bounding box around detected objects in images by combining several pre-constructed deep neural networks. 
%We remark that the unsupervised ensemble approach developed in this paper %is applicable to a broad range of real world scenarios beyond those described here.   

The regression problem we consider in this paper is a particular instance of unsupervised ensemble learning. 
Motivated in part by crowdsourced labeling tasks,  
 previous works on unsupervised ensemble learning mostly focused on \textit{discrete} outputs, considering binary, multiclass or ordinal classification 
\cite{Johnson1996,sheng2008get,whitehill2009whose,raykar2010learning,platanios2014estimating,platanios2016estimating,zhou2012learning}. \citet{dawid1979maximum} were among the first to consider the problem of unsupervised ensemble classification. Their approach was based on the assumption that experts make independent errors conditioned on the unobserved true class label.
Even for this simple model, estimating the experts' accuracies and the unknown  labels via maximum likelihood is a non-convex problem, typically solved by the expectation-maximization algorithm. Recently, several authors proposed spectral and tensor based methods that are computationally 
efficient and asymptotically consistent  \cite{anandkumar2014tensor,zhang2014spectral,jaffe2015estimating}.

In contrast to the discrete case, nearly all previous works on  ensemble regression considered only the supervised setting. Some   ensemble methods, such as boosting and random
forest are widely used in practice. 

In this work we propose a framework to study unsupervised ensemble regression, focusing on linear aggregation methods. In  Section \ref{sec:setup}, we first review the optimal weights that minimize the mean squared error (MSE) and highlight the key quantities that need to be estimated in an unsupervised setting. Next, in Section \ref{sec:priorart} we describe related prior work in supervised and unsupervised regression. 

Our main contributions appear in Section \ref{sec:unsupervised_ensemble_regression}. We propose a framework for unsupervised ensemble regression, based on an analogue of the Dawid and Skene classification model, adapted to the regression setting. Specifically, we assume that the \(m\) experts make approximately uncorrelated errors with respect to the optimal predictor that minimizes the MSE. 
We show that if we knew the minimal attainable MSE, then under our assumed model, the accuracies of the experts  can be consistently estimated by solving a system of linear equations.  
Next, based on our theoretical analysis, we develop methods to estimate this minimal MSE, detect the best and worst regressors and derive U-PCR, a novel unsupervised principal components ensemble regressor. 
  Section \ref{sec:experiments}  illustrates our methods  and the improved accuracy of U-PCR\ over the ensemble mean and median, on a variety of regression problems.
These include both problems for which we trained multiple regression algorithms, as well as the two applications mentioned above where the regressors were constructed by a third party and only their predictions were given to us. 

Our main findings are that given only the predictions
$f_i(x_j)$ and the first two moments of the response: (i) our approach is able to distinguish between 
hard prediction problems where any linear aggregation
of the $m$ regressors yields large errors, and feasible problems where
a suitable linear combination of the regressors can accurately estimate the response; (ii) our ranking method is able to reliably detect the most and least accurate experts; and (iii) quite consistently, U-PCR performs as well as and sometimes significantly better than the mean and median of the $m$ regressors. We conclude in Section \ref{sec:summary} with a summary and  future research directions  in unsupervised ensemble regression.

\section{Problem Setup}
\label{sec:setup}

%Let $(X,Y)$ be a pair of random variables, 
Consider a regression problem
with a continuous response \(Y\in\mathbb{R}\) and explanatory features  $X$ from an instance space $\mathcal X$. Let \(\lbrace f_1,\ldots,f_m \rbrace \) be \(m\) pre-constructed regression functions, $f_i:\mathcal X\to\mathbb{R}$, interchangeably also called experts, and let \(\{x_j\}_{j=1}^n\) be  \(n\) i.i.d. samples from the marginal distribution of $X$. We consider the following unsupervised ensemble regression setting, in which the only observed data is the \(m\times n\)
matrix of predictions  
\begin{align}
\begin{pmatrix}
f_1(x_1) & \cdots & f_1(x_n) \\ 
\vdots & \ddots & \vdots  \\ 
f_m(x_1) & \cdots & f_m(x_n) \end{pmatrix} .
        \label{def:data_matrix} 
\end{align}
In particular, there are no labeled data pairs $(x_j, y_j)$ and no a-priori knowledge on the accuracy of the $m$ regressors.

Given only the matrix (\ref{def:data_matrix}), and explicit knowledge of the first two moments of \(Y\), we ask whether  it is possible to: (i) estimate the accuracies of the \(m\) experts, or at least identify the best and worst of them, and (ii) accurately estimate the responses \(y_{j}\) by an ensemble method $\hat{y}: \{f_i(x) \}_{i=1}^m \mapsto \R$, whose input are the predictions of $f_1, \dots, f_m$. As we explain below knowing the first two moments of $Y$ seems necessary as otherwise the data matrix (\ref{def:data_matrix}) can be arbitrarily shifted and scaled. Such knowledge is reasonable in various settings, for example from past experience, previous observations or physical
principles.     
 
Following the literature on supervised ensemble regression, we consider {\em linear} ensemble learners. Specifically, we 
restrict ourselves to the following subclass
\begin{equation}
        \hat y_\w (x) = \muY + \sum_{i=1}^m w_i \big( f_i(x) - \mu_i \big) \label{eq:unbiased_ensemble_learner}
\end{equation}
where $\muY=\E[Y]$ and $\mu_i = \E[f_i(X)]$  are assumed known, and $\mathbf w = (w_1,\ldots,w_m)^T$. Note that in this subclass, for any vector $\mathbf w$, $\E_{(X,Y)}\big[\hat y_\w(X)\big] = \muY$. While $\mu_i$ is typically unknown, it can be accurately estimated given the predictions of $f_i$ in Eq. (\ref{def:data_matrix}) and provided $n\gg 1$. 

As our risk measure, we use the popular mean squared error 
$\mbox{MSE}=\E[(Y-\hat y(X))^2]$. For completeness, we first review the optimal weights under this risk and describe several supervised ensemble methods that estimate them.

\paragraph{Optimal Weights.}
Let $C$ be the $m \times m$ covariance matrix of the $m$ regressors with elements 
\begin{equation}
        C_{ij} = \E[(f_i(X) - \mu_i)(f_j(X) - \mu_j)]\ ,
                \label{def:C}
\end{equation}
and let $\boldsymbol \rho = (\rho_1,\ldots,\rho_m)^T$ be the vector of covariances between the individual regressors and the true response,
\begin{equation}
        \rho_i = \E_{(X,Y)}[(Y-\muY)(f_i(X) - \mu_i)]\ .
                \label{def:rho}
\end{equation} 
 Let $\w^*$ be a weight vector that minimizes the MSE
\begin{equation}
        \w^* = \arg\!\min_\w\ \E_{(X,Y)}\Big[\big(\hat y_\w(X) - Y\big)^2\Big]
                \label{eq:w*}        
\end{equation}
Then it is easy to show that: 

\begin{lemma} \label{lemma:w*}
The weights $\w^*$ satisfy
\begin{equation}
        \boldsymbol \rho = C\w^*.
                \label{eq:p=Cw*}
\end{equation}
\end{lemma}
Note that $\w^*$ depends only on $\boldsymbol \rho$ and $C$. 
If the $m$ ensemble regressors are linearly independent, then 
$C$ is invertible and $\w^*$ is unique. In our unsupervised scenario, the matrix $C$ can be estimated from the predictions $f_i(x_j)$.
In contrast, estimating $\boldsymbol \rho$ directly from its definition in Eq. (\ref{def:rho}) requires labeled data. A key challenge in unsupervised ensemble regression is thus to estimate $\boldsymbol \rho$ without any labeled data.

%%%%%%%%%%%%%%%%%%%%%%% SECTION %%%%%%%%%%%%
\section{Previous Work}\label{sec:priorart}

This section provides a brief overview of prior art, first methods for unsupervised ensemble regression, and then two \emph{supervised} ensemble regression methods that are related to our approach. We conclude this section with the popular Dawid-Skene model of unsupervised ensemble classification, also relevant to our work.     

%%%%%%%%%%%%%%%%%%%%%%% SUBSECTION %%%%%%%%%%%%
\subsection{Unsupervised Ensemble Regression}

Whereas many works considered unsupervised ensemble classification, far fewer studied the regression case.  \citet{donmez2010unsupervised}, proposed a general framework called unsupervised-supervised learning. In the case of regression, they assumed that the marginal probability density function of the response $p(y)$ is known and that the regressors follow a \emph{known} parametric model with parameter $\theta$. In this setup, given only unlabeled data,  $\theta$ can be estimated  by maximum likelihood. 
In contrast, our approach is far more general as we do not assume a  parametric model, nor  knowledge of the full marginal density $p(y)$.

More closely related is the recent work of
 \citet{wu2016spectral}, which in turn is based on \citet{ionita2016spectral} and \cite{Parisi2014}.
 Here, the authors compute the leading eigenvector of the covariance of the \(m\) regressors, and use it both to detect inaccurate regressors and to determine the weights of the accurate ones.
However, as \citet{wu2016spectral} themselves write, this relation between the leading eigenvector and regressor accuracy  ``is based on intuition, and we do not have a rigorous mathematical proof so far". Our work provides a solid theoretical support for a variant of this spectral approach.

%%%%%%%%%%%%%%%%%%%%%%% SUBSECTION %%%%%%%%%%%%
\subsection{Supervised Ensemble Regression}

As reviewed by \citet{ensemblesurvey}, quite a few supervised ensemble regressors were proposed over the past 30 years. These can be 
broadly divided into two groups. Methods in the first group re-train 
a basic regression algorithm multiple times on different subsets of the labeled data, possibly also assigning  weights to the various labeled instances. Examples include stacking \cite{wolpert1992stacked,Breiman1996,Leblanc1996},
random forest \cite{breiman2001random_forests} and boosting \cite{freund1995desicion,friedman2000additive}.

In contrast, 
ensemble methods in the second group view the regressors as \emph{pre-constructed} and only estimate the weights of their linear combination.    \citet{perrone1992networks} and  \citet{Merz1999} derived two such methods, which we briefly describe below.  

While not directly related, there is also extensive literature on supervised combination of forecasts in time series analysis and on methods to combine multiple estimators, see \citet{timmermann2006forecast,lavancier2016general} and many references therein.

%%%%%%%%%%%%%%%%%%%%%%% SUBSECTION %%%%%%%%%%%%
\subsection{Generalized Ensemble Method}\label{ss:GEM}

\citet{perrone1992networks} were among the first to consider supervised ensemble regression. They defined the {misfit}  of predictor $i$ as $m_i(x) = f_i(x) - y$, and proposed the  Generalized Ensemble Method (GEM),  with $\sum_i w_i = 1,$
$$
        \hat y_{\mbox{\tiny GEM}}(x) = \sum_i w_i f_i(x) = y + \sum_i w_i m_i(x).
$$
 The corresponding weights that minimize the MSE are 
\begin{align}
        w_i^{\mbox{\tiny GEM}} = {\sum_j C^{*-1}_{ij} \Big/
        \sum_{j,k} C^{*-1}_{jk}} .
                \label{eq:PCweights}
\end{align}
%
%\begin{align}
%        w_i^{\mbox{\tiny GEM}} = \frac{\sum_j C^{*-1}_{ij}}
%        {\sum_j \sum_k C^{*-1}_{jk}} \ .
%                \label{eq:PCweights}
%\end{align}
where $C^*$ is the $m \times m$ misfit population covariance matrix 
\begin{equation}
                \label{def:C*}
        C^*_{ij} = \E_{(X,Y)}[m_i(X) m_j(X)] \ .
\end{equation}
\citet{perrone1992networks} proposed to estimate the unknown matrix $C^*$ and consequently $\mathbf w^{\mbox{\tiny GEM}}$ using a labeled set  \(\{(x_i,y_i)\}_{i=1}^{n_{\mbox{\tiny train}}}\). Unfortunately, in many practical scenarios multi-colinearity between the $m$ regressors leads to an ill conditioned
matrix $C^*$, that cannot be robustly inverted. 

%%%%%%%%%%%%%%%%%%%%%%% SUBSECTION %%%%%%%%%%%%
\subsection{PCR*}

A common approach to handle ill conditioned multivariate problems is via  principal component regression \cite{jolliffe2002principal}. 
In the context of supervised ensemble learning, \citet{Merz1999} suggested such a method, denoted PCR*.  Given a labeled set \(\{(x_i,y_i)\}_{i=1}^{n_{\mbox{\tiny train}}}\) 
let $\hat{C}$ be the \(m\times m\) sample covariance matrix of the \(m\) regressors, 
\begin{align*}
        \hat C_{ij} &= \frac{1}{n_{\text {train}}}\sum_{k=1}^{n_{\text{train}}} \big(f_i(x_k) - \hat \mu_i\big)\big(f_j(x_k) - \hat \mu_j\big) 
\end{align*}
where $\hat \mu_i = \frac{1}{n_{\mbox{\tiny train}}}\sum_{j=1}^{n_{\mbox{\tiny train}}}f_i(x_j)$, and let  $\mathbf v_1, \ldots, \mathbf v_K$ be the top 
\( K\) leading eigenvectors of  $\hat C$. 
\citet{Merz1999}
proposed a weight vector of the form $\w = \sum_{k=1}^K a_k \mathbf v_k$, with coefficients $a_k$ determined by least squares regression over the training set. The number of principal components $K$ is chosen by minimizing $V$-fold cross validation error. 

In the common scenario where some ensemble regressors are highly correlated,  the matrix $C^*$ is ill-conditioned. The GEM estimator, which inverts \(\hat C^* \) then yields unstable predictions. In contrast,
PCR* with a small number of components can be viewed as a regularized method, providing stability and robustness. 
In a supervised setting, \citet{Merz1999} found PCR* to outperform GEM. 

%%%%%%%%%%%%%%%%%%%%%%%%%%%%%%%%%
\subsection{Unsupervised Ensemble Classification}

The simplest model for unsupervised ensemble classification, going back to  \citet{dawid1979maximum} is that conditional on the label $Y$, classifiers make independent errors
\begin{equation}
        \Pr \big( f_i(X),f_j(X)  | Y \big) = \Pr( f_i(X) | Y ) \cdot \Pr ( f_j(X)  | Y ).
        \label{eq:DS} 
\end{equation}
\citet{dawid1979maximum} estimated the classifier accuracies and the labels by the EM\ method. In recent years several authors developed computationally efficient and rate optimal methods to estimate these quantities \cite{anandkumar2014tensor,zhang2014spectral,jaffe2015estimating}.

To the best of our knowledge, our work is the first to propose an analogue of this assumption to the regression case, rigorously study it, and consequently derive corresponding unsupervised ensemble regression schemes.

%%%%%%%%%%%%%%%%%%%%%%% SECTION %%%%%%%%%%%%
\section{Unsupervised Ensemble Regression}\label{sec:unsupervised_ensemble_regression}

Given only the predictions $f_i(x_j)$,  
the simplest unsupervised approach to estimate the response \(y\) 
%denoted as the {\em Basic Ensemble Method} by \citet{perrone1992networks},
at an instance \(x\) is to average the \(m\) regressors, 
 $$
        \hat y^{\mbox{\tiny AVG}}(x) = \frac{1}{m} \sum_{i=1}^m f_i(x)\ .
$$
Averaging is the optimal linear estimator when all regressors make independent zero-mean errors of equal variance. 
A more robust but non-linear  method is the median, 
$$
        \hat y^{\mbox{\tiny MED}}(x) = median\big(f_1(x),\ldots,f_m(x)\big).
$$
Averaging and  median are na\"ive estimators in the sense that prediction at each  $x_{k}$ depends only on \(f_{i}(x_k)\) and does not depend on the  other observations $f_i(x_j)$, \(x_{j}\neq x_{k}\).

As we show theoretically below and illustrate empirically in Section \ref{sec:experiments}, under some reasonable assumptions, one can do significantly better than the ensemble mean and median by analyzing all the data $f_i(x_j)$ and  in particular the \(m\times m \) covariance matrix \(C\) of the \(m\)
regressors. 

Specifically, we propose a novel framework to study unsupervised ensemble regression, based on the assumption that the \(m\) experts make approximately uncorrelated errors with respect to the optimal predictor. We develop methods to detect the best and worst regressors and derive U-PCR, a novel unsupervised principal components ensemble regressor. 
Similar to \citet{Merz1999}, the weight vector of U-PCR is a linear combination of the top few eigenvectors of $C$ (typically just one or two). The key novelty is that we   estimate
the coefficients in a fully unsupervised manner. 

To this end, we do assume knowledge of the first two moments of $Y$. Such knowledge seems inevitable, as otherwise the observed data may be arbitrarily shifted and scaled without changing the correlation of the regressors. Knowing the first moment  $\muY=\E[Y]$, allows to  estimate the bias $b_i=\E[f_i(X) - Y]$ of each regressor $f_i$ by its mean over the $n$ unlabeled samples,
$$
\hat b_i = \frac{1}{n}\sum_{j=1}^n f_i(x_j) - \muY = b_i + o_P(1).
$$ 
Knowledge of $\var(Y)$ allows a rough estimate of the accuracy of the $m$ regressors. A very accurate regressor must have $\var(f_i)\approx \var(Y)$, whereas if \(\var(f_i)\ll \var(Y)\) or
$\var(f_i)\gg \var(Y)$, then $f_i$ must have a large error. 

In what follows we consider predicting the mean-centered responses \(y_j-\theta_1\) by a linear combination of the mean centered predictors, $\hat y(x) = \sum_j w_j (f_j(x)-\hat{b}_j-\muY)$. 
We thus work with the mean centered matrix
%For future use we denote the mean-centered data by  
$$
        Z_{ij} = f_i(x_j) - \hat b_i \ -\muY.
$$
This is equivalent to assuming that $\E[f_i(X)] = \E[Y] = 0$. 

%%%%%%%%%%%%%%%%%%%%%%% SUBSECTION %%%%%%%%%%%%    
\subsection{Statistically Independent Errors}

As discussed in Section \ref{sec:setup}, in light of the optimal weights in Eq. (\ref{eq:p=Cw*}), the key challenge in unsupervised ensemble regression is to estimate the vector \(\rb\) of Eq. (\ref{def:rho}), without any labeled data. 

To this end, we propose the following regression analogue of the Dawid-Skene assumption of conditionally independent experts.
Recall that when the risk function is the MSE, the optimal regressor is the conditional mean, 
$$
g(x)=\E[Y|X=x].
$$  
Its mean is $g_1=\E[g(X)]=\E[Y]=0$, and its MSE is
$
\E[(Y-g(X))^2]=\var(Y) -g_2,
$
where 
\begin{equation}
g_2=\E_{X}[g(X)^2] = \E_{(X,Y)}[g(X)Y] .
        \label{eq:g_2}
\end{equation}
For each regressor,  write $f_i(x)=g(x)+h_i(x)$. Since $f_i$ is mean centered, $E[h_i(X)]=0$. Hence, \(\rho_i=\E[f_i(X)Y]\) simplifies to 
\begin{equation}
\rho_i %= \E[f_i(X)Y]
=
\E[g(X)Y]+\E[{h_i(X)}Y] = g_2 + a_i\,.
                \label{eq:rho_g_ai}
\end{equation}
Similarly, the MSE of regressor $i$ is 
\begin{equation}
\mbox{MSE}(f_i) %=\E[(f_i(X)\!-Y)^2]
     =g_2-2 a_i+\E[h_i(X)^2]. 
        \label{eq:MSE_f}
\end{equation}

In this notation, the challenge is thus to estimate \(g_{2}\) and the vector  $\ab = (a_1,\ldots,a_m)$. Inspired by Eq. (\ref{eq:DS}) in the case of classification, we assume the $m$ regressors make {\em independent errors  with respect to $g(X)$}, namely that
\footnote{Strictly speaking, the assumption is that the deviations from $g(X)$ are uncorrelated and not necessarily independent. } 
\begin{equation}
\E[h_i(X)h_j(X)]=0.
        \label{eq:ind_h}
\end{equation} 
This assumption is reasonable, for example,  when the \(m\) regressors 
were trained independently and are rich enough to well approximate the conditional mean \(g(X)\). Note that when the response \(Y\) is perfectly predictable from the  features \(X\), then \(g(X)=Y\) and our assumption then states that the $m$ regressors make independent errors with respect to the response \(Y\). This can be viewed as the regression equivalent of the Dawid-Skene model in classification. 

Next, we consider how to estimate the values $a_i$ under 
the independent error assumption of Eq. (\ref{eq:ind_h}). Suppose for a moment that the value of \(g_{2}\) of Eq. (\ref{eq:g_2}) was known. We shall discuss how to estimate it in the next section. As the following theorem shows, in this case, we can consistently estimate $\rb$ by solving a system of linear equations. \begin{theorem} \label{thm:independent_misfits}
Assume that the given $m \ge 3$ regressors  make pairwise independent errors with respect to the conditional mean. If \(g_{2}\) is known then 
given only the data matrix $Z$, we can consistently estimate the vector $\rb$ at rate $O_P(1/\sqrt{n})$.
\end{theorem}
\begin{proof}
It is instructive to first consider the population setting where $n\to\infty$.
Here, under the  assumption (\ref{eq:ind_h}), the off-diagonal entries of the population covariance are 
\begin{equation}
C_{ij}=\E[f_i(X)f_j(X)] = g_2 + a_i + a_j
        \label{eq:C_g_a}
\end{equation}   
Since $C$ is symmetric, these off-diagonal entries provide ${m \choose 2}$ linear equations for the \(m\) unknown variables $\ab~=~(a_1,\ldots,a_m$). Thus, if  $m \ge 3$ there are enough linearly independent equations to uniquely recover $\ab$.
The vector \(\rb\) can then be computed from Eq. (\ref{eq:rho_g_ai}). 

In practice, the population matrix \(C\) is unknown. However, given the 
$m\times n$ matrix $Z$, we may estimate it by the sample covariance $\hat C$. Since $\hat C_{ij} = C_{ij} + O_P(\tfrac{1}{\sqrt{n}})$, estimating $(a_1,\ldots,a_m)$ by least-squares yields a consistent estimator $\hat{\boldsymbol \rho}$ with asymptotic error $O_P(1/\sqrt{n})$.
\end{proof}

%\begin{remark}\label{remark:error_rates_ind}
%For large values of \(m\) there are many more equations than unknowns, and the linear system (\ref{eq:C_g_a}) is extremely well conditioned.
%It can be shown that 
%$$
%        \E \tfrac{1}{m} \| \hat{\mathbf a} - \mathbf a \|^2 = 
%        O\left(\tfrac1{mn}\right).
%$$
%\end{remark}

\begin{remark}
In practice, assumption (\ref{eq:ind_h}) that all \(m\) regressors make
independent errors, may be strongly violated at least for some pairs. To be robust to deviations from this assumption one may choose a suitable loss function $L(\cdot)$, and solve the optimization problem
\begin{equation}
\hat{\ab}
=\argmin_{(a_1,\ldots,a_m)} \sum_{i<j}
L(\hat{C}_{ij}-g_2 - a_i - a_j).
        \label{eq:a_L}
\end{equation} 
In our experiments, we considered both the absolute loss and the standard squared loss. 
\end{remark}

%%%%%%%%%%%%%%%%%%%%%%% SUBSUBSECTION %%%%%%%%%%%%
\subsection{Unsupervised PCR}\label{ss:PCR}

 The analysis above assumed knowledge of \(g_{2}\), or equivalently of the minimal attainable MSE\ of the regression problem at hand. Clearly, this would seldom be known to the practitioner. Further, any guess of \(g_{2} \in [0,\var(Y)]\) gives a valid solution. Specifically, let $\hat{\ab}(q)$ be 
the solution of (\ref{eq:a_L}) with an assumed value \(g_{2}=q\). 
Then, due to the additive structure inside the parenthesis in Eq. (\ref{eq:a_L}), regardless of the loss function $L$, we have 
$\hat{\ab}(q) = \hat{\ab}(0)-\frac{q}2 {\1}$ where 
$\1=(1,\ldots,1)^T\in\mathbb{R}^m$. Similarly,  by Eq. (\ref{eq:rho_g_ai}),
\begin{equation}
\hat{\rb}(q) = \hat{\rb}(0)+\frac{q}2\, {\1}\,.
        \label{eq:rho_q}
\end{equation}
What is needed is thus a {\em model selection} criterion that would be able to accurately estimate the value of $g_2$, given the family of possible solutions $\hat{\rb}(q)$ for $0\leq q\leq\var(Y)$. 

To motivate our proposed estimator of   \(g_2\), let us first analyze the model of the previous section, but with the additional assumption that all $m$ regressors are fairly close to the optimal conditional mean \(g(x)\). Namely, for analysis purposes, we scale the deviations \(h_i\) by a parameter \(\epsilon\), 
\begin{equation}
f_i(x)=g(x) + \epsilon h_i(x)
        \label{eq:f_epsilon}
\end{equation}
and study the behaviour of various quantities as a function of $\epsilon$. Specifically, under Eq. (\ref{eq:f_epsilon}), the population covariance of the $m$ regressors takes the form
\[
C(\epsilon) = g_2 \1 \1^T+\epsilon(\ab\1^T + \1 \ab^T)+\epsilon^2 D
\]
where $a_i=\E[h_i(X)Y]$ and $D$ is a diagonal matrix with entries $D_{ii}=\E[h_i^2(X)]$.
The following lemma characterizes the leading eigenvalue and eigenvector of \(C\), as \(\epsilon\to 0\).

\begin{lemma} \label{lemma:lambda_v_C}
Let $\lambda_{1}(\epsilon), \mathbf v_{1}(\epsilon)$ be the largest eigenvalue and corresponding  eigenvector of $C(\epsilon)$. Then, as $\epsilon\to 0$,  
\begin{eqnarray}
        \label{eq:lambda_epsilon}
       \lambda_{1}(\epsilon) &= & g_2 m +(2\ab^T \1)\cdot \epsilon  + O(\epsilon^2) \\
       {\bf v}_{1}(\epsilon) & = & g_{2}\1 +  (\ab -\frac{\ab^T\1}m\1)\cdot\epsilon +  O(\epsilon^2).
        \label{eq:v_epsilon}
\end{eqnarray}
\end{lemma}

Several insights can be gained from this lemma. First, at 
$\epsilon=0$ the matrix $C(\epsilon=0)=g_2 \1 \1^T$ is rank one with a single non-zero eigenvector ${\bf v}_{1}=\1$ and corresponding eigenvalue $\lambda_{1}=g_2 m$. Hence, if the \(m\) regressors are all very close  to \(g(x)\), 
their population matrix $C$ is nearly rank one and very ill conditioned. Even with an accurate estimate of $g_2$ and consequently of $\hat\rb$, inverting Eq. (\ref{eq:p=Cw*}) to estimate  $\hat{\bf w}=(\hat C)^{-1}\hat\rb$ would then be extremely unstable. 

Second, under the model (\ref{eq:f_epsilon}), $\rb=g_2\1+\epsilon\ab$. Comparing this  to Eq. (\ref{eq:v_epsilon}),   the vector $\rb$ and the leading eigenvector ${\bf v}_1$, properly scaled, are \textit{nearly identical}, up to a small shift by $(\frac1m\sum a_i) \epsilon$ and up to $O(\epsilon^2)$ terms. Moreover, up to \(O(\epsilon^2)\) terms, the matrix \(C(\epsilon)\) has rank two, spanned by the two vectors \(\1\) and \(\ab\).
Hence, up to $O(\epsilon^2)$ terms, the true vector $\rb$ can be written as a linear combination of the first two eigenvectors of $C$.
Thus, even though the matrix $C$ is ill conditioned, a principal component approach, with just $K=1$ or 2 components, can provide an excellent approximation of the optimal weight vector $\bf w^*$.  While our focus is on unsupervised ensemble, this analysis provides a rigorous theoretical support
for the PCR* method of \citet{Merz1999}, a result which may be of independent interest for supervised ensemble learning.

Third, since by Eq.(\ref{eq:MSE_f}), \(\mbox{MSE}(f_i)=g_2-2a_i \epsilon  +O(\epsilon^2)\), the worst and best regressors may be detected by the largest and smallest entries in  ${\bf v}_1$ or in the estimated vector $\hat\rb$. 
%In the next section we show that this approach works quite well in practice.  

Lemma \ref{lemma:lambda_v_C} suggests several ways to estimate the unknown quantity \(g_2\). By Eq. (\ref{eq:lambda_epsilon}), one option is  \(\hat g_2=\lambda_{1}/m\). Under our assumed model, this would incur an error $(2\sum a_j)\epsilon+O(\epsilon^2)$. Another option, which we found works better in practice is to consider the relation between $\hat{\rb}(q)$ and the top eigenvector ${\bf v}_{1}$ of $C$, normalized to 
$\|{\bf v}_1\|=1$. Specifically, we  estimate $g_2$ by minimizing the following residual,  
\begin{equation}
\hat g_2 \!=\!\argmin_{q\in[0,\var(Y)]} \mbox{RES}(q) \!=\! 
\argmin_{q} 
\frac{\|\hat{\rb}(q)-\left({\bf v}_1^T\hat{\rb}(q)\right){\bf v}_1\|}
{\|\hat{\rb}(q)\|}
        \label{eq:g2_hat}
\end{equation}
where \(\hat\rb(q)\) is given in Eq. (\ref{eq:rho_q}). From the estimate  $\hat g_2$, the weight vector 
of U-PCR is
%in our linear ensemble estimator is then 
\begin{equation}
{\bf w}^{\mbox{\tiny U-PCR}} = \frac{1}{\lambda_1}({\bf v}_1^T\hat\rb(\hat g_2)){\bf\, v}_1
        \label{eq:w_UPCR}
\end{equation} 

A sketch of our proposed scheme appears in Algorithm \ref{alg:main_flow}. 
% % % % % % % % % % % % % % % % % % ALGORITHMS % % % % % % % % %
\begin{algorithm}[tb]
   \caption{Sketch of U-PCR }
   \label{alg:main_flow}
\begin{algorithmic}
        \STATE {\bfseries Input:} Predictions $f_i(x_j), \E[Y]$ 
              and $\var(Y)$ 
        %\STATE {\bfseries Input:} First two moments of response $Y$
        \STATE Compute covariance \(\hat{C}\) and its leading eigenvector ${\bf v}_1$ 
        \STATE For $q\in[0,\var(Y)]$, compute $\hat\rb(q)$ by  Eqs. (\ref{eq:rho_g_ai}), (\ref{eq:a_L}) and (\ref{eq:rho_q})
        \STATE Estimate $g_2$ via Eq. (\ref{eq:g2_hat}). 
        \STATE Set $\rb=\hat\rb(\hat g_2)$ and $\rho_{\max}=\max \rho_i$
        \IF{$\hat g_2 < \epsilon_{L}\cdot\var(Y)$} 
        \STATE Difficult prediction problem; STOP
        \ENDIF
        \STATE Exclude experts with $\rho_i<0.05\,\var(Y)$ 
               or $\rho_i<\rho_{\max}/3$ 
        %\STATE  $\mathcal I:= \texttt{Select-Experts}\big(\hat\rb(\hat g_2)\big)$
        %\STATE Choose covariance sub-matrix $C\big|_{\mbox{\tiny INL}}$
        \STATE Recalculate $\mathbf v_1, \hat{\boldsymbol\rho}(q),\hat g_2 $              on remaining experts
        \STATE {\bf Output:} Weight vector $\hat{\mathbf w}$ of Eq. (\ref{eq:w_UPCR})
        \end{algorithmic}
\end{algorithm}

\begin{figure*}[t]
    \subfigure[CCPP: Accurate prediction possible]
        {\includegraphics[width=.32\textwidth]{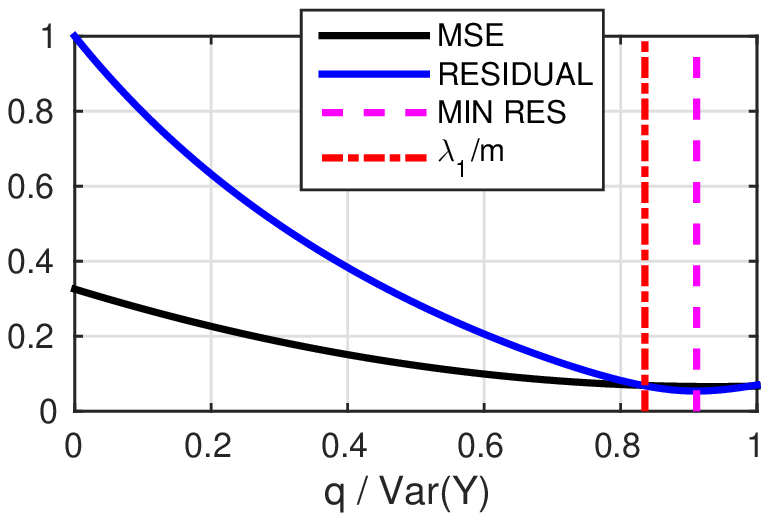}} %
    \subfigure[Basketball: Challenging task]{\includegraphics[width=.32\textwidth]{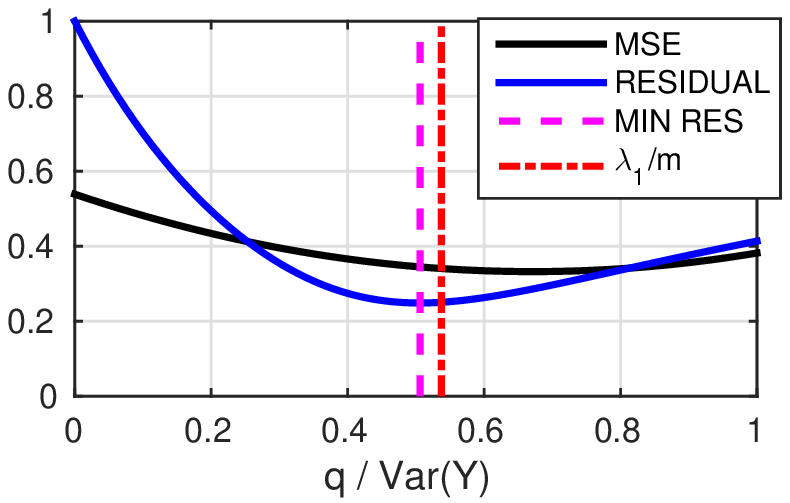}} %
    \subfigure[Affairs: Limited information on response]{\includegraphics[width=.32\textwidth]{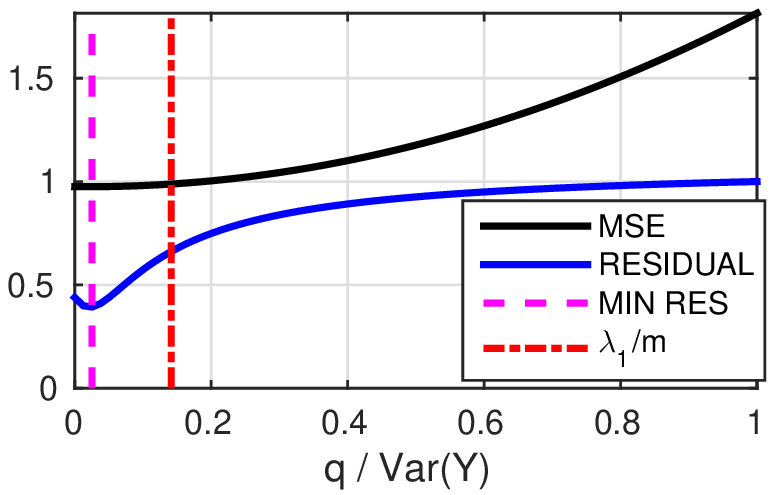}}
    \caption{Plots of \(\mbox{MSE}(q)\) and \(\mbox{RES}(q)\) for problems of  easy, moderate and hard difficulty levels.  The vertical lines are the estimated $\hat g_2$ at which the residual is minimal and \(\lambda_1/m\).}
    \label{fig:g2_estimation}
\end{figure*}

%%%%%%%%%%%%%%%%%%%%%%%%%%%%%%%%%%%%
\subsection{Practical Issues}

Before illustrating the competitive performance of U-PCR, we discuss several important   
practical issues that need to be addressed when handling real-world  ensembles, whose individual regressors may not satisfy our assumptions. 

First, when the true value $g_2\ll\var(Y)$, the regression problem at hand is very difficult, and no linear combination of the \(m\) predictors can give a small error. If our estimated $\hat g_2/\var(Y)\le \epsilon_L$ for some small threshold $\epsilon_L$, say 0.1, this is an indication of such a difficult problem. In this case we stop and do not attempt to construct an ensemble learner. 

Second, even when accurate prediction is possible, in our experience, if some regressors are far less accurate than others, then it is important to detect them and {\em exclude} them from the ensemble, and recompute the various quantities after their removal. However, in the rare cases that after this removal  only \(m\leq 4\) regressors remained, then we found it better to compute their simple average instead of Eq. (\ref{eq:w_UPCR}).  

Finally,  if the second eigenvalue is not
extremely small, then it is beneficial to project the vector \(\hat{\rb}\) onto the first two eigenvectors of \(\hat C\) .  In our experiments
we did so when  $\lambda_2>0.1 \cdot \mbox{Trace}(\hat C)$. Then, Eq. (\ref{eq:w_UPCR})
is replaced by 
\[
{\bf w}^{\mbox{\tiny U-PCR}} = 
\frac{1}{\lambda_1}({\bf v}_1^T\hat\rb(\hat g_2)){\bf\, v}_1 + 
\frac{1}{\lambda_2}({\bf v}_2^T\hat\rb(\hat g_2)){\bf\, v}_2\,.  
\]

%%%%%%%%%%%%%%%%%%%%%%%%%%%%%%%%%%%%%%%%%%%%%%%%%%%%%%%%%%
\section{Experiments}\label{sec:experiments}

We illustrate the performance of U-PCR on various real world regression problems. These include problems for which we trained multiple regression algorithms, and  two applications where the regressors were constructed by a third party and only their predictions were given to us.

We compare U-PCR to the ensemble mean and median as well as to a linear oracle regressor of the form (\ref{eq:unbiased_ensemble_learner}), which has access to all the response values $y_{j}$. It determines its weights by ordinary least squares over all $n$ samples 
\begin{equation} \label{eq:linear_ensemble_oracle}
        {\bf w}_{\mbox{\tiny or}} = (Z Z^T)^{-1} Z \cdot (\mathbf y - \muY) \ .
\end{equation}
We denote the normalized MSE 
of the oracle 
by $\delta_{\mbox{\tiny or}} = \mbox{MSE}({\bf w}_{\mbox{\tiny or}})/\var(Y)$.

We divide the regression problems into three difficulty levels:   (i) $\dor \lesssim 0.1$, where accurate prediction is possible by a linear combination of the $m$ regressors; (ii) $0.1 \lesssim \dor \lesssim 0.8$, a challenging regression task; and (iii) $\dor \gtrsim 0.8$, where the \(m\) experts provide very little, if any information on $Y$. 

We start with the following basic question: Given only $f_i(x_j)$ and the first two moments of $Y$, can we roughly estimate the difficulty level of our problem? If it belongs to level (i) or (ii), is it possible to detect the most accurate or least accurate regressors? Finally, can we construct a linear combination at least as accurate as the mean or median?

%\paragraph{GEM-Oracle Definition.} We further define the GEM-Oracle estimator with weights $\w^{\mbox{\tiny GEM}}$ determined by Eq. (\ref{eq:uGEM}), such that $\boldsymbol{\rho}$ is determined by oracle access to the true response. This oracle is of the subclass of estimators of the form (\ref{eq:unbiased_ensemble_learner}) which are restricted to have weights that sum to $1$. As shown in lemma \ref{lemma:uGEM}, the GEM-Oracle achieves the lowest MSE within that subclass.

\begin{figure*} 
        \centering
            \subfigure[Before outlier removal]{\includegraphics[width=.24\linewidth]{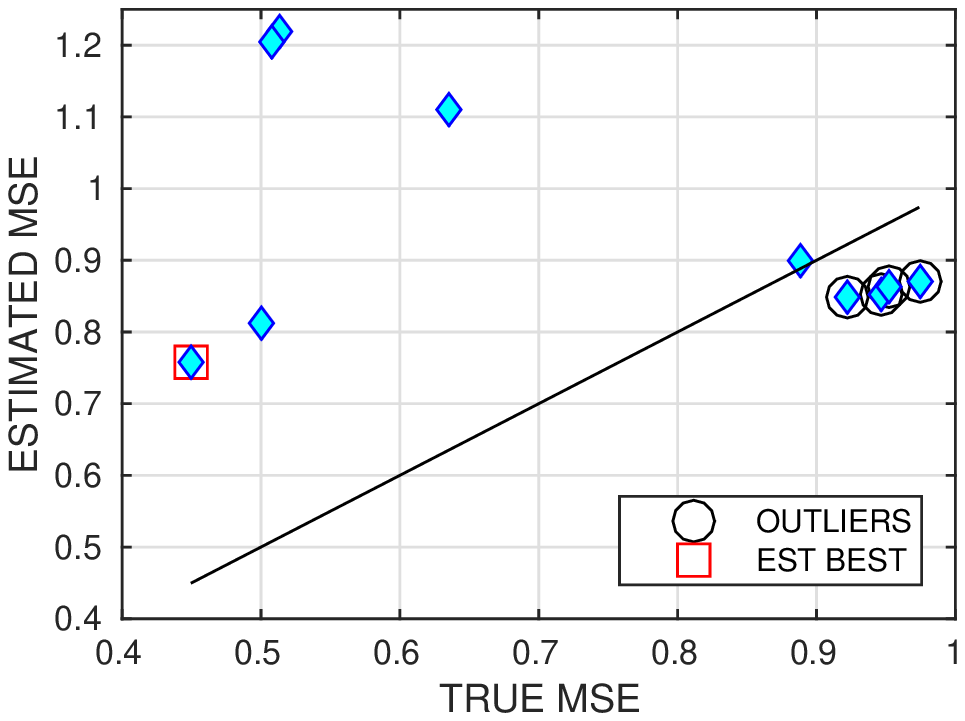}} 
            \subfigure[After outlier removal] {\includegraphics[width=.24\linewidth]{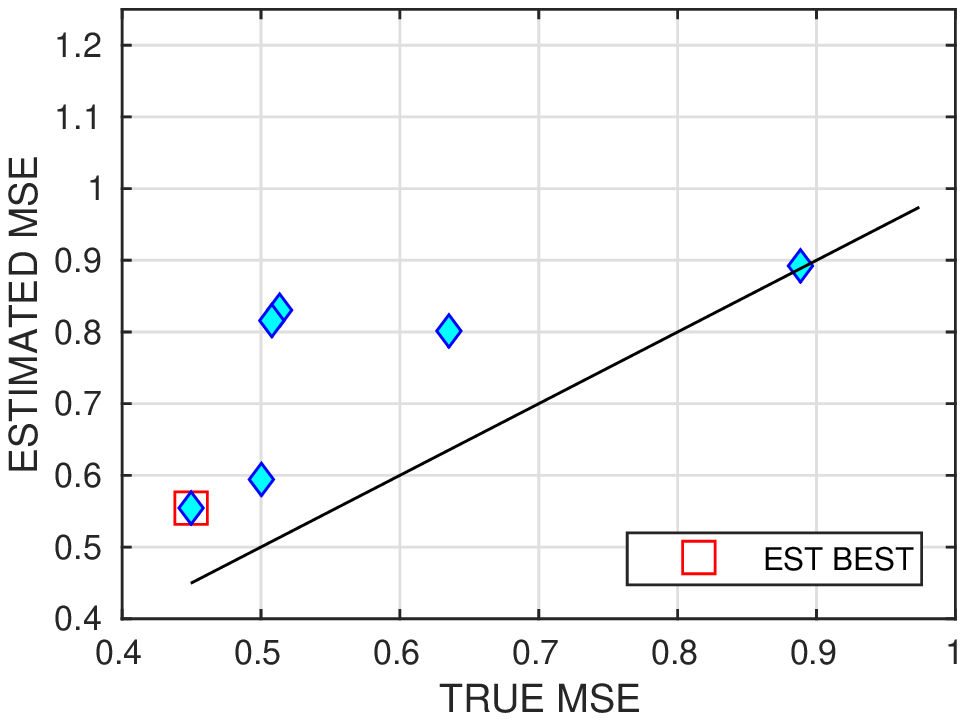}}         
        \subfigure[Before outlier removal]{\includegraphics[width=.24\linewidth]{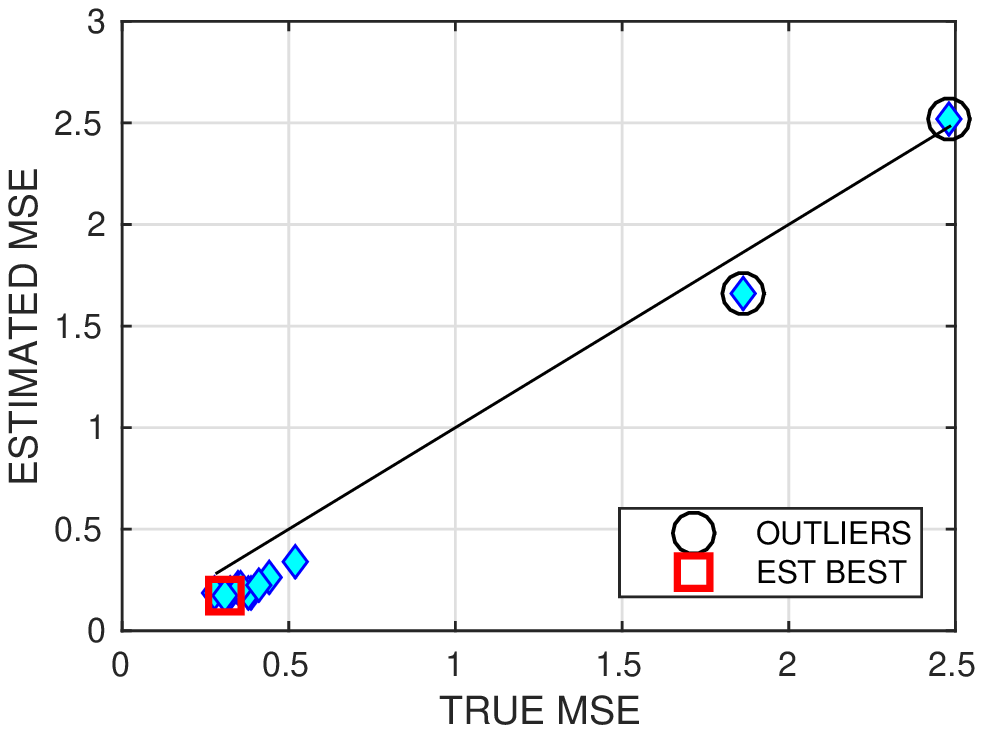}}
        \subfigure[After outlier removal] {\includegraphics[width=.24\linewidth]{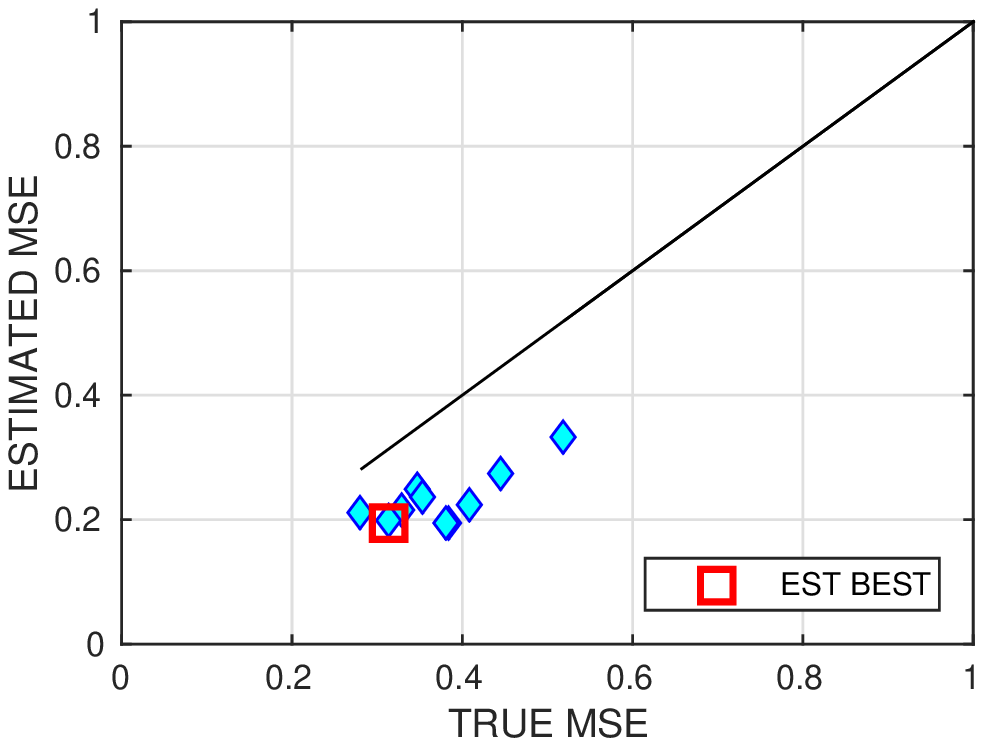}}             
        \caption{Estimated MSE(\(f_{i})\) vs. true MSE for 
                        Flights AUS (left panels) and  UACC812 protein (right panels) before and after outlier removal. The outlier removal scheme is not based on the estimated MSE, but rather as described in main text, on the entries of the estimated $\hat\rb$. In some datasets, such as Flights AUS, recalculation after this removal gives more accurate estimates of regressors' MSE. }
                \label{fig:outlier_detection} 
\end{figure*}

\subsection{Manually Crafted Ensembles}
With precise details appearing in the supplement, we considered  18 different prediction tasks, including energy output prediction in a power plant, flight delays, basketball scoring and more. Each dataset was randomly split into \(n_{\mbox{\tiny train}}\) samples used to train 10 different regression algorithms and remaining \(n\) samples to construct the observations \(f_i(x_j)\), see Table \ref{table:datasets}
in Supplementary. The regressors included
 Ridge Regression, SVR, Kernel Regression and Decision Trees, among others. 

Table \ref{table:MSE_ENSEMBLE} in the supplement shows the MSE of U-PCR, mean and median averaged over 20 repetitions, each with different random splits into train and test samples. On several datasets, U-PCR obtained a significantly lower MSE.   
With further details in the supplement, here we highlight some of  our key results. We start by estimating $g_2$ and classifying the problems by difficulty level. Fig. \ref{fig:g2_estimation} shows this estimation procedure on three datasets. The $x$-axis is the  value of \(q$ normalized by $\var(Y).\)
The black curve is the unobserved MSE$(q)$ obtained by the weight vector of Eq. (\ref{eq:w_UPCR}), with assumed
$\hat\rb(q)$. The red curve is the computed residual $\mbox{RES}(q)$ of Eq. (\ref{eq:g2_hat}) and the vertical line is the estimated \(\hat g_2\). Our approach is indeed able to correctly detect the difficulty levels of these problems and estimate a value \(\hat g_2\), whose corresponding MSE\ is not too far from the minimal achievable by using any of the \(\hat\rb(q)\).  
Fig. \ref{fig:rho_estimation} in the supplement shows the  estimated $\hat\rb$ vs. the true \(\rb\). For easy problems with $g_2\ll\var(Y)$ the agreement is remarkable. 

Next, we evaluated the ability to detect the most accurate regressor in the ensemble. We measured the excess risk of selecting the regressor with the smallest estimated MSE,  compared to the best regressor, which is unknown. Additionally, we measured the excess risk in selecting the single regressor with the greatest corresponding entry in the leading eigenvector of $C$. The details are given in Table \ref{table:excess_risk} of the Supplementary. Our experiments show that in most cases choosing the predictor with lowest estimated MSE outperforms the one with largest entry in ${\bf v}_1$.

 Fig. \ref{fig:outlier_detection} shows  the effectiveness of detecting inaccurate regressors, by pruning those whose entries $\hat\rho_i < \rho_{\max}/3$ or $\hat\rho_i < 0.05\var(Y)$. Finally, Fig. \ref{fig:cat1_comparison} illustrates the advantages of U-PCR over the mean and median, on problems of easy to moderate difficulty.

\begin{figure*}[ht] 
        \centering
        \subfigure[Nearly Perfect Prediction]
                {\includegraphics[width=.4\linewidth]{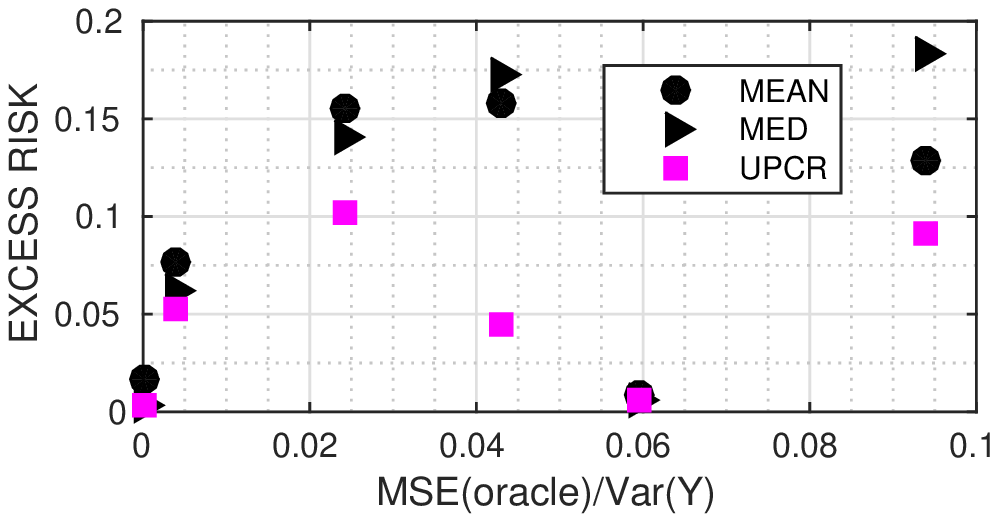}}         \subfigure[Challenging Problems]{\includegraphics[width=0.4\linewidth]{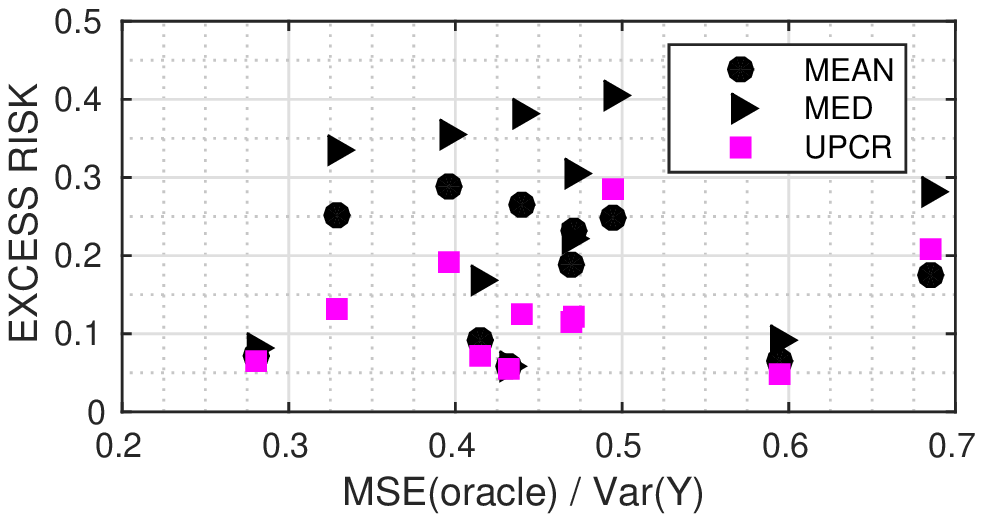}}
        \caption{Excess risk MSE(ensemble)$-$MSE(oracle), divided by $\var(Y)$ for easy problems (left) and challenging ones (right).}
                \label{fig:cat1_comparison}  
\end{figure*}

% % % % % % % % % % % % % % % % SUBSECTION % % % % % % %
%
%
\subsection{HPN-DREAM Challenge Experiment}
Next, we consider real world problems where the ensemble regressors were constructed by a third party. The first problem came from the HPN-DREAM breast cancer network inference challenge \cite{hill2016inferring}. Here, participants were asked to predict the time varying concentrations of 4 proteins after the introduction of an inhibitor. We were given the predictions of \(m=12\) models on \(n\approx 2500\) instances.  We constructed a separate U-PCR model for each protein. Fig. \ref{fig:outlier_detection} demonstrates the success of our method  in detecting accurate regressors and removing inaccurate ones.
 Fig. \ref{fig:dream8_accuracy} shows that U-PCR outperformed the mean and median on 3 of the 4 proteins.
We note that for all four proteins, the single best model had comparable MSE to U-PCR, however, this model is unknown. For three of the four proteins U-PCR had smaller MSE than that of the single model estimated as being the most accurate.   

\begin{figure}[h] 
        \centering
        \includegraphics[width=.7\linewidth]{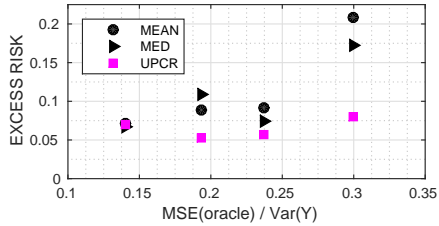} 
        \caption{HPN-DREAM Challenge Accuracy.}
                \label{fig:dream8_accuracy}  
\end{figure}

%%%%%%%%%%%%%%%%%%%%%%% SUBSECTION %%%%%%%%%%%%
\subsection{Bounding Box Experiment}

Here we were given the predictions of
6 deep learning models trained by Seematics Inc., 
on the location of physical objects in images. The models were trained on the PASCAL Visual Object Classes dataset  \cite{pascal-voc-2012}, whereas the predictions were made on images from COCO\ dataset \cite{lin2014microsoft}.
We focused on three object classes \{person, dog, cat\}, with each neural network providing four coordinates for the bounding box: $(x_1,y_1)$ and $(x_2,y_2)$. We used U-PCR as an ensemble predictor each coordinate separately, with the mean squared error as out measure of accuracy. An example of the MSE estimation by our method can be seen in Fig. \ref{fig:bbox_mse_estimation}, and the accuracy for all object classes and all coordinates in Fig.  \ref{fig:bbox_accuracy}.
Results on few   images   are in the supplementary. 
\begin{figure}[ht] 
        \centering
        \includegraphics[width=.7\linewidth]{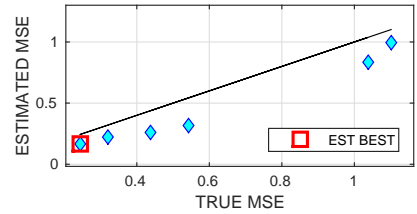} 
        \caption{MSE estimation for class Cat, coordinate $x_1$}
                \label{fig:bbox_mse_estimation}  
\end{figure}
\begin{figure}[ht] 
        \centering
        \includegraphics[width=.7\linewidth]{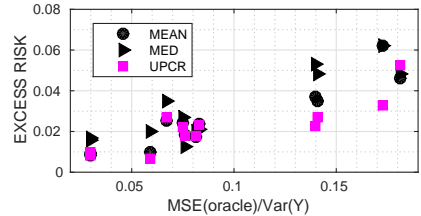} 
        \caption{Bounding box prediction accuracies}
                \label{fig:bbox_accuracy}  
\end{figure}

%%%%%%%%%%%%%%%%%%%%%%% SUBSECTION %%%%%%%%%%%%
\section{Summary and Discussion}\label{sec:summary}

In this paper we tackled the problem of unsupervised ensemble regression. We presented a framework to explore this problem, based on an independent error assumption. We proposed methods, together with theoretical support, to  detect the best and worst regressors and to linearly aggregate them, all in an unsupervised manner. As our theoretical analysis in Section \ref{sec:unsupervised_ensemble_regression} showed, unsupervised ensemble regression is different from the well studied
problem of unsupervised ensemble classification, and  required different approaches to its solution.

Our work raises several questions. One of them is how
to extend our method  to a semi-supervised setting, in which there is also a limited amount of labeled data. 
It is also interesting to theoretically understand the relative benefits
of labeled versus unlabeled data for ensemble learning. 

Another direction for future research is to replace the strict independent error assumption by  more complicated yet realistic models for dependencies between the regressors. In the context of
unsupervised classification, \citet{fetaya2016unsupervised} relaxed
the conditional independence model of Dawid and Skene by  introducing an intermediate layer of latent variables. Instead of a rank-one off diagonal covariance, the matrix \(C\) in their model had a low rank structure, which the authors learned by a spectral method. It is interesting whether a similar approach can be
developed for an ensemble of regressors. 

\section*{Acknowledgments}
The authors thank Moshe Guttmann and the Seematics team for their
help. This research was funded by the Intel Collaborative
Research Institute for Computational Intelligence (B.N.)
and by NIH grant 1R01HG008383-01A1 (Y.K. and B.N.).

\bibliography{mybib}
\bibliographystyle{apalike}

%\end{document}

\newpage

\appendix

\onecolumn
%%%%%%%%%%%%%%%%%%%%%%% SECTION %%%%%%%%%%%%
%\section{Supplementary Material} \label{sec:supp}
%%%%%%%%%%%%%%%%%%%%%%% SUBSECTION %%%%%%%%%%%%
%
\centerline{\Large \bf Supplementary Material}

\section{Proofs}

\begin{proof}[Proof of Lemma \ref{lemma:lambda_v_C}]
The proof follows a perturbation approach similar to the one outlined in \citet{nadler2008finite}. Since $C(\epsilon)$ is symmetric and quadratic in $\epsilon$, classical results on perturbation theory \cite{Kato} imply that in a small neighborhood of $\epsilon=0$, the leading eigenvalue and eigenvector are analytic in $\epsilon$. We may thus expand them in a Taylor series, 
\begin{eqnarray*}
\lambda(\epsilon)&=&\lambda_0 + \lambda_1\epsilon  + \lambda_2\epsilon^2 + \ldots \\
\bf v(\epsilon)& = & {\bf v}_0 + {\bf v}_1 \epsilon + {\bf v}_2 \epsilon^2 + \ldots 
\end{eqnarray*}
We insert this expansion into the eigenvector equation $C(\epsilon){\bf v}(\epsilon)=\lambda(\epsilon){\bf v}(\epsilon)$ and solve the resulting equations at increasing powers of $\epsilon$. 

The leading order equation reads $g_2\1\1^T{\bf v}_0 = \lambda_0{\bf v}_0$, which gives ${\bf v}_0\propto \1$ and $\lambda_0=g_2 \|\1\|^2=g_2m$. Since the eigenvector ${\bf v}(\epsilon)$ is defined only up to a multiplicative factor, we conveniently chose it to be that  $\1^T{\bf v}(\epsilon)=g_2 m$ holds for all $\epsilon$. 
This gives ${\bf v}_0=g_2\1$ and ${\bf v}_1^T{\bf v}_0=0$. 

The $O(\epsilon)$ equation reads
\begin{equation}
g_2 \1\1^T {\bf v}_1 + (\ab \1^T+\1\ab^T) {\bf v}_0 = \lambda_0 {\bf v}_1 + \lambda_1 {\bf v}_0.
        \label{eq:O_epsilon}
\end{equation}
Multiplying this equation from the left by ${\bf v}_0^T$ gives
\[
2({\bf v}_0^T\1)(\ab^{T}{\bf v}_0) = \lambda_1 \|{\bf v}_0\|^2  
\]
or $\lambda_1=2\sum a_j$. Thus, Eq. (\ref{eq:lambda_epsilon}) follows. 
Inserting the expression for $\lambda_1$ back into Eq. (\ref{eq:O_epsilon}) gives
\[
{\bf v}_1 = \frac1{\lambda_0}[(\ab^T{\bf v}_0)\1 + (\1^T{\bf v}_0)\ab -(2\sum_j a_j) {\bf v}_0]
\]
from which Eq. (\ref{eq:v_epsilon}) readily follows.
\end{proof}

%%%
\section{Datasets \& Results}

\begin{table*}[t]
\vskip 0.1in % This spacing is required by ICML instructions
\caption{Mean squared error of different ensemble methods, normalized by $\var(Y).$ On the Affairs data U-PCR estimates it is a difficult problem and does not predict outcomes. Numbers in bold represent cases where one of the unsupervised ensemble regressors was significantly better than the others. }
\label{table:misc_results}
\vskip 0.1in % This spacing is required by ICML instructions
\begin{center}
\begin{small}
\begin{sc}
\begin{tabular}{|l||c|c|c|c|}\hline \hline
Dataset                 & Oracle        & U-PCR & Mean & Median \\ \hline
             Abalone &   0.43 ($\pm 0.01$) &     0.49 ($\pm 0.01$) &     0.49 ($\pm 0.01$)&       0.49 ($\pm 0.01$)  \\ \hline
             Affairs &   0.92 ($\pm 0.00$) &     N.A. &     0.96 ($\pm 0.01$)&       0.94 ($\pm 0.00$)  \\ \hline
          Basketball &   0.28 ($\pm 0.01$) &     0.35 ($\pm 0.01$) &     0.35 ($\pm 0.00$)&       0.36 ($\pm 0.00$)  \\ \hline
        Bike sharing &   0.00 ($\pm 0.00$) &     0.00 ($\pm 0.00$) &     0.02 ($\pm 0.00$)&       0.00 ($\pm 0.00$)  \\ \hline
       Blog feedback &   0.41 ($\pm 0.03$) &     0.49 ($\pm 0.02$) &     0.50 ($\pm 0.02$)&       0.58 ($\pm 0.02$)  \\ \hline
                Ccpp &   0.06 ($\pm 0.00$) &     0.07 ($\pm 0.00$) &     0.07 ($\pm 0.00$)&       0.07 ($\pm 0.00$)  \\ \hline
         Flights AUS &   0.33 ($\pm 0.04$) &     {\bf 0.46} ($\pm 0.07$) &     0.58 ($\pm 0.06$)&       0.66 ($\pm 0.08$)  \\ \hline
         Flights BOS &   0.47 ($\pm 0.04$) &     \bf 0.58 ($\pm 0.08$) &     0.66 ($\pm 0.03$)&       0.69 ($\pm 0.08$)  \\ \hline
         Flights BWI &   0.44 ($\pm 0.06$) &     \bf 0.56 ($\pm 0.09$) &     0.71 ($\pm 0.03$)&       0.82 ($\pm 0.08$)  \\ \hline
         Flights HOU &   0.40 ($\pm 0.09$) &     \bf 0.59 ($\pm 0.07$) &     0.69 ($\pm 0.03$)&       0.75 ($\pm 0.08$)  \\ \hline
         Flights JFK &   0.50 ($\pm 0.05$) &     0.78 ($\pm 0.21$) &     \bf 0.74 ($\pm 0.03$)&       0.90 ($\pm 0.04$)  \\ \hline
         Flights LGA &   0.47 ($\pm 0.04$) &     \bf 0.59 ($\pm 0.06$) &     0.70 ($\pm 0.03$)&       0.78 ($\pm 0.09$)  \\ \hline
    Flights longhaul &   0.69 ($\pm 0.05$) &     0.89 ($\pm 0.24$) &     0.86 ($\pm 0.06$)&       0.97 ($\pm 0.01$)  \\ \hline
           Friedman1 &   0.02 ($\pm 0.00$) &     \bf 0.13 ($\pm 0.00$) &     0.18 ($\pm 0.01$)&       0.16 ($\pm 0.01$)  \\ \hline
           Friedman2 &   0.00 ($\pm 0.00$) &     \bf 0.06 ($\pm 0.01$) &     0.08 ($\pm 0.01$)&       0.07 ($\pm 0.01$)  \\ \hline
           Friedman3 &   0.04 ($\pm 0.01$) &     \bf 0.09 ($\pm 0.02$) &     0.20 ($\pm 0.02$)&       0.22 ($\pm 0.02$)  \\ \hline
       Online videos &   0.09 ($\pm 0.01$) &     \bf 0.18 ($\pm 0.01$) &     0.22 ($\pm 0.01$)&       0.28 ($\pm 0.02$)  \\ \hline
  Wine quality white &   0.60 ($\pm 0.01$) &     \bf 0.64 ($\pm 0.01$) &     0.66 ($\pm 0.01$)&       0.69 ($\pm 0.01$)  \\ \hline
\end{tabular}
\end{sc}
\end{small}
\end{center}
\vskip -0.1in
        \label{table:MSE_ENSEMBLE}
\end{table*}

\subsection{Selecting a Single Regressor}

Table \ref{table:excess_risk} compares the MSE of the single regressor estimated to be the most accurate, versus  the MSE\ of the single best regressor, which is unknown in this setup. The following two methods were compared: (i) Selecting the regressor with the maximum entry $\hat\rho_i$, and (ii) selecting the regressor with the minimal estimated MSE. The experiments were repeated 20 times for each dataset, mean and standard deviations are reported. All values are normalized by $\var(Y)$ for fair comparison.

\begin{table*}[t]
\vskip 0.1in % This spacing is required by ICML instructions
\caption{MSE of the single best estimated regressor}
\label{table:excess_risk}
\vskip 0.1in % This spacing is required by ICML instructions
\begin{center}
\begin{small}
\begin{sc}
\begin{tabular}{|l||c|c|c|c|}\hline \hline
\textbf{Dataset} & \textbf{Oracle MSE}  & \textbf{Best Regressor MSE} & \textbf{MSE} of $\argmin_i \widehat{\mbox{MSE}_i} $  & \textbf{MSE} of $\argmax_i \hat\rho_i $ \\\hline
             Abalone &   0.43 ($\pm 0.01$) &     0.45 ($\pm 0.01$) &     \textbf{0.49} ($\pm 0.03$)&      0.77 ($\pm 0.21$)  \\ \hline
          Basketball &   0.28 ($\pm 0.01$) &     0.32 ($\pm 0.01$) &     \textbf{0.36} ($\pm 0.01$)&      0.49 ($\pm 0.13$)  \\ \hline
        Bike sharing &   0.00 ($\pm 0.00$) &     0.00 ($\pm 0.00$) &     0.00 ($\pm 0.00$)&       0.00 ($\pm 0.00$)  \\ \hline
       Blog feedback &   0.41 ($\pm 0.03$) &     0.43 ($\pm 0.03$) &     0.66 ($\pm 0.02$)&       0.62 ($\pm 0.19$)  \\ \hline
                CCPP &   0.06 ($\pm 0.00$) &     0.07 ($\pm 0.00$) &     \textbf{0.07} ($\pm 0.00$)&      0.09 ($\pm 0.02$)  \\ \hline
         Flights AUS &   0.33 ($\pm 0.04$) &     0.48 ($\pm 0.03$) &     0.56 ($\pm 0.09$)&       0.66 ($\pm 0.24$)  \\ \hline
         Flights BOS &   0.47 ($\pm 0.04$) &     0.53 ($\pm 0.04$) &     \textbf{0.61} ($\pm 0.13$)&      1.04 ($\pm 0.19$)  \\ \hline
         Flights BWI &   0.44 ($\pm 0.06$) &     0.50 ($\pm 0.05$) &     \textbf{0.50} ($\pm 0.05$)&      0.87 ($\pm 0.45$)  \\ \hline
         Flights HOU &   0.40 ($\pm 0.09$) &     0.52 ($\pm 0.06$) &     \textbf{0.53} ($\pm 0.08$)&      1.13 ($\pm 0.45$)  \\ \hline
         Flights JFK &   0.50 ($\pm 0.05$) &     0.54 ($\pm 0.05$) &     \textbf{0.64} ($\pm 0.16$)&      0.95 ($\pm 0.22$)  \\ \hline
         Flights LGA &   0.47 ($\pm 0.04$) &     0.53 ($\pm 0.03$) &     \textbf{0.59} ($\pm 0.12$)&      1.04 ($\pm 0.42$)  \\ \hline
    Flights longhaul &   0.69 ($\pm 0.05$) &     0.74 ($\pm 0.05$) &     \textbf{0.79} ($\pm 0.11$)&      1.17 ($\pm 1.04$)  \\ \hline
           Friedman1 &   0.02 ($\pm 0.00$) &     0.03 ($\pm 0.00$) &     0.24 ($\pm 0.04$)&       \textbf{0.03} ($\pm 0.00$)  \\ \hline
           Friedman2 &   0.00 ($\pm 0.00$) &     0.01 ($\pm 0.00$) &     0.14 ($\pm 0.02$)&       \textbf{0.02} ($\pm 0.03$)  \\ \hline
           Friedman3 &   0.04 ($\pm 0.01$) &     0.07 ($\pm 0.01$) &     0.25 ($\pm 0.21$)&       \textbf{0.14} ($\pm 0.04$)  \\ \hline

       Online videos &   0.09 ($\pm 0.01$) &     0.10 ($\pm 0.01$) &     0.34 ($\pm 0.02$)&       \textbf{0.17} ($\pm 0.03$)  \\ \hline
  Wine quality white &   0.60 ($\pm 0.01$) &     0.62 ($\pm 0.01$) &     \textbf{0.77} ($\pm 0.01$)&      1.12 ($\pm 0.19$)  \\ \hline
\end{tabular}
\end{sc}
\end{small}
\end{center}
\vskip -0.1in
\end{table*}

\subsection{Dataset Descriptions}
Below is a list of the prediction tasks for which we manually trained ensembles with 10 regressors. Table \ref{table:datasets} summarizes the main characteristics of each dataset, and Table \ref{table:misc_results} contains the mean squared errors of the different approaches normalized by $\var(Y)$. The experiments were repeated 20 times, and the mean and standard deviations are reported. We used standard Python packages for the regression algorithms with the following parameters: Ridge ($\alpha=0.5$), Kernel Regression (kernel chosen using cross validation between polynomial, RBF, sigmoid), Lasso ($\alpha=0.1$), Orthogonal Matching Pursuit, Linear SVR ($C=1$), SVR with RBF kernel ($C$ chosen using cross validation out of 0.01, 0.1, 1, 10), Regression Tree (depth 4), Regression Tree (infinite depth), Random Forest (100 trees), and a Bagging Regressor.

\begin{table}[t!]
        \begin{center}
        \begin{small}
  \begin{threeparttable}
        \vskip 0.1in % This spacing is required by ICML instructions
        \caption{Prediction Problems}
        \label{table:datasets}
        \vskip 0.1in % This spacing is required by ICML instructions
        \begin{sc}

        \begin{tabular}{lrrrrcc}
                        \bf Name & \multicolumn{1}{c}{$n$} & $n_{\mbox{\tiny train}}$ &  \multicolumn{1}{c}{$d$}  & $\overline{\mbox{MSE}}(f)$ & $\min_i \mbox{MSE}(f_i)$ & $\mbox{MSE}_{\mbox{\tiny oracle}}$ \\
                \hline \hline \abovespace
             Abalone &      3277 &          700 &          7 &   0.59 &  0.45 &   0.431 ($\pm 0.006$) \\ 
             Affairs &      5466 &          700 &          7 &   1.08 &  0.93 &   0.922 ($\pm 0.004$) \\ 
          Basketball &     48899 &          900 &          9 &   0.43 &  0.32 &   0.281 ($\pm 0.005$) \\ 
        Bike Sharing &     15579 &         1600 &         16 &   0.07 &  0.00 &   0.000 ($\pm 0.000$) \\ 
       Blog Feedback &     24197 &        28000 &        280 &   0.64 &  0.43 &   0.415 ($\pm 0.026$) \\ 
                CCPP &      8968 &          400 &          4 &   0.10 &  0.07 &   0.059 ($\pm 0.001$) \\ 
         Flights AUS &     47595 &         1000 &         10 &   0.76 &  0.48 &   0.329 ($\pm 0.035$) \\ 
         Flights BOS &    112705 &         1000 &         10 &   0.84 &  0.53 &   0.470 ($\pm 0.042$) \\ 
         Flights BWI &    101665 &         1000 &         10 &   0.85 &  0.50 &   0.440 ($\pm 0.065$) \\ 
         Flights HOU &     53044 &         1000 &         10 &   0.87 &  0.52 &   0.397 ($\pm 0.094$) \\ 
         Flights JFK &    113960 &         1000 &         10 &   0.89 &  0.54 &   0.495 ($\pm 0.051$) \\ 
         Flights LGA &    111911 &         1000 &         10 &   0.86 &  0.53 &   0.471 ($\pm 0.040$) \\ 
   Flights Long Haul &      9393 &         1000 &         10 &   1.00 &  0.73 &   0.686 ($\pm 0.051$) \\ 
           Friedman1 &     18800 &         1000 &         10 &   0.31 &  0.03 &   0.024 ($\pm 0.001$) \\ 
           Friedman2 &     19400 &          400 &          4 &   0.17 &  0.01 &   0.004 ($\pm 0.001$) \\ 
           Friedman3 &     19400 &          400 &          4 &   0.35 &  0.07 &   0.043 ($\pm 0.006$) \\ 
       Online Videos &     66484 &         2100 &         21 &   0.34 &  0.10 &   0.094 ($\pm 0.006$) \\ 
  \belowspace  Wine Quality White &         3598 &         1100 &         11 &    0.79 &  0.62 &  0.595 ($\pm 0.011$) \\ \hline

     \end{tabular}
     \end{sc}
    \begin{tablenotes}
      \footnotesize
      \item $n$ is the number of held-out samples. The input $X$ is $d$ dimensional, and the same $n_{\mbox{\tiny train}}$ random samples were used to train the different algorithms in the ensemble. $\overline{\mbox{MSE}}(f)$ is the average regressor error, $\min_i\mbox{MSE}(f_i)$ is the minimal error achieved by a regressor in the ensemble, and $\mbox{MSE}_{\mbox{\sc \tiny oracle}}$ is the MSE of the oracle, normalized by $\var(Y)$, with its standard deviation in parenthesis. For each dataset the split between train and test was performed 20 times, averages are listed. 
    \end{tablenotes}
  \end{threeparttable}
     \end{small}
     \end{center}
\end{table}

\begin{figure}
    \subfigure[CCPP: Accurate prediction possible]
        {\includegraphics[width=.32\textwidth]{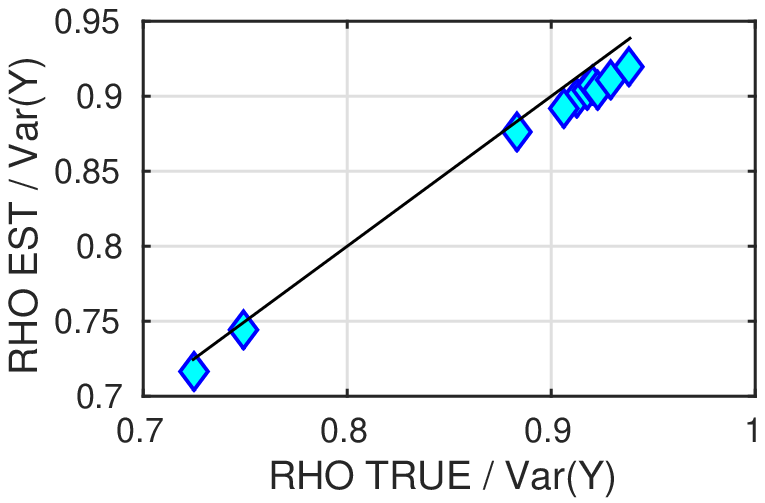}} %
    \subfigure[Basketball: Challenging task]{\includegraphics[width=.32\textwidth]{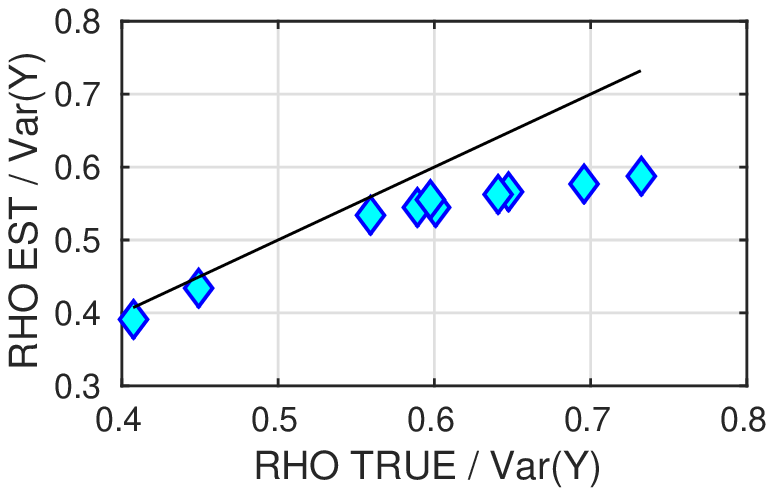}} %
    \subfigure[Affairs: No information on response]{\includegraphics[width=.32\textwidth]{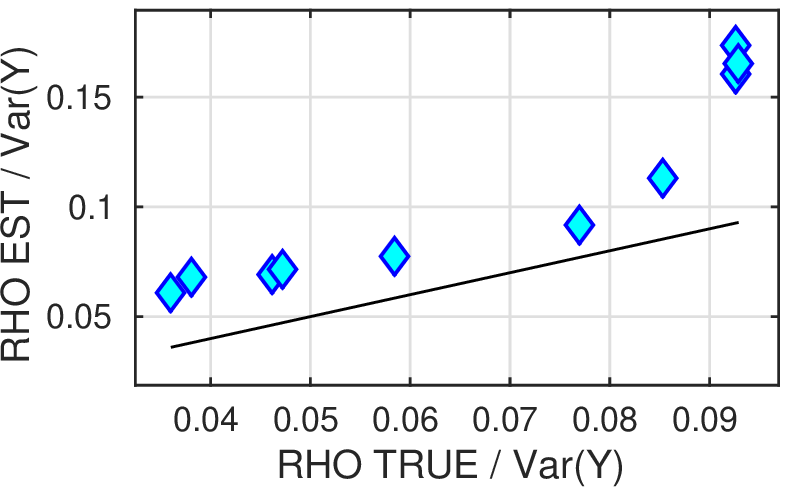}}
    \caption{Estimated $\hat\rb$ vs. true $\rb$ in three regression problems of different difficulty levels}
    \label{fig:rho_estimation}
\end{figure}

\paragraph{Abalone.}
A dataset containing features of abalone, where the goal is to predict its age \cite{Lichman:2013}.
\textcolor{blue}{archive.ics.uci.edu/ml/datasets/Abalone}

\paragraph{Affairs.}
A dataset containing features describing an individual such as time at work, time spent with spouse, and time spent with a paramour. The goal here is to predict the time spent in extramarital affairs.
\textcolor{blue}{statsmodels.sourceforge.net/0.6.0/datasets/generated/fair.html}

\paragraph{Basketball.}
Dataset contains stats on NBA players. Task: Predict number of points scored by the player on the next game.
The features are: \texttt{name}, \texttt{venue}, \texttt{team}, \texttt{date}, \texttt{start}, \texttt{pts\_ma}, \texttt{min\_ma}, \texttt{pts\_ma\_1}, \texttt{min\_ma\_1, pts}, 
where \texttt{start} is whether or not the player started, \texttt{pts} is number of points scored, \texttt{min} is number of minutes played, \texttt{ma} stands for moving average, starts at season, and \texttt{ma\_1} is a moving average with a 1 game lag.

\paragraph{Bike Sharing.}
Bike sharing service statistics, including weather and seasonal information \cite{fanaee2014event}. The prediction task here is the daily and hourly count of bikes rented. \textcolor{blue}{archive.ics.uci.edu/ml/datasets/Bike+Sharing+Dataset}

\paragraph{Blog Feedback.} 
Instances in this dataset contain features extracted from blog posts. The task associated with the data is to predict how many comments the post will receive.

\paragraph{Flights.}
Information on flights from 2008, where the task is to predict the delay upon arrival in minutes.
The features here are the date, day of the week, scheduled and actual departure times, scheduled arrival times, flight ID, tail number, origin, destination, and distance. Due to its size, we split this dataset to flights originating from specific airports (AUS, BOS, BWI, HOU, JFK, and LGA), and long-haul flights.
\textcolor{blue}{stat-computing.org/dataexpo/2009/the-data.html}

\paragraph{CCPP.}
Combined Cycle Power Plant UCI-dataset containing physical characteristics such as temperature and humidity. The task here is to predict the net hourly electrical energy output of the plant.
\textcolor{blue}{archive.ics.uci.edu/ml/datasets/Combined+Cycle+Power+Plant}

\paragraph{Friedman \#1}
Motivated by \citet{Breiman1996}, we used simulated data according to \citet{friedman1991multivariate}. The predictor variables $x_1,\ldots,x_5$ are independent and uniformly distributed over $[0,1]$. The response is  
$$
        y=10\sin(\pi x_1 x_2) + 20 (x_3 -.5)^2 + 10x_4 + 5x_5 + \epsilon
$$
and $\epsilon \sim \mathcal N(0,1)$. 

\paragraph{Friedman \#2}
The second data set tested by \citet{friedman1991multivariate} simulated impedance in an alternating current circuit. Here four predictor variables $x_1,\ldots,x_4$ are uniformly distributed over the ranges $[0,100], [40\pi,560\pi], [0,1]$ and $[1,11]$ respectively. The response was
$$
        y = \sqrt {x_1^2 + (x_2 x_3 - (1/x_2x_4))^2} + \epsilon_2
$$
with $\epsilon_2 \sim \mathcal N(0,\sigma_2^2)$, where the variance was chosen to provide a 3-to-1 signal to noise ratio. 
For the third dataset in this series Friedman \#3, see the original paper \cite{friedman1991multivariate}.

\paragraph{Online Videos.}
YouTube video transcoding dataset. Predict the transcoding time based on parameters of the video.
\textcolor{blue}{archive.ics.uci.edu/ml/datasets/Online+Video+ Characteristics+and+Transcoding+Time+Dataset}

\paragraph{Wine Quality White.}
Predict the quality score (1-10) of white wine based on chemical characteristics, such as acidity and pH level \cite{cortez2009modeling}.
\textcolor{blue}{archive.ics.uci.edu/ml/datasets/Wine+Quality}

\begin{figure}
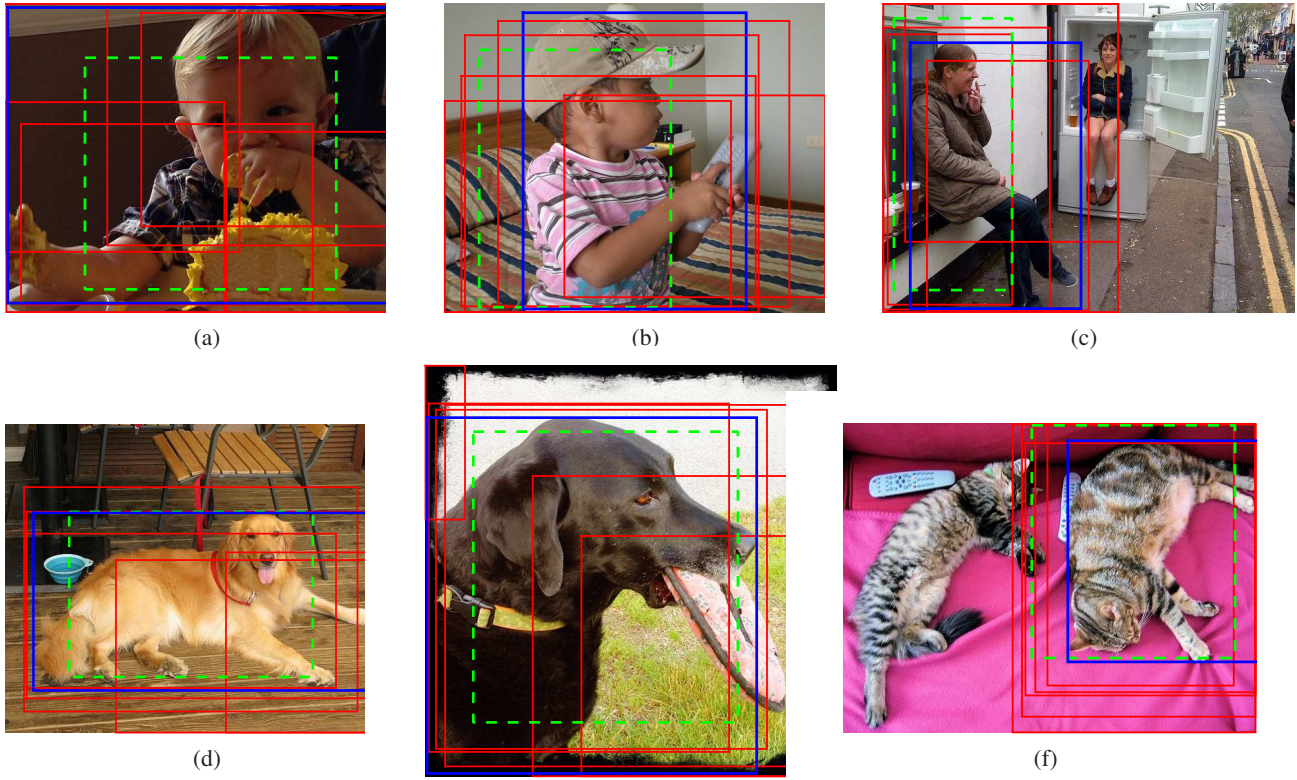

    \subfigure[]{\includegraphics[width=.32\textwidth]{person321214.eps}} %
    \subfigure[]{\includegraphics[width=.32\textwidth]{person245460.eps}} %
    \subfigure[]{\includegraphics[width=.32\textwidth]{person24243.eps}}
    \subfigure[]{\includegraphics[width=.32\textwidth]{dog256628.eps}} %
    \subfigure[]{\includegraphics[width=.32\textwidth]{dog121443.eps}} %
    \subfigure[]{\includegraphics[width=.32\textwidth]{cat39769.eps}}
    \caption{Sample images from the bounding box experiment. The ground truth bounding box is shown in blue, U-PCR in dashed green, and the regressors are shown in red.}
    \label{fig:bbox}
\end{figure}

\end{document}